\documentclass{article}

\usepackage[nonatbib, preprint]{utils/neurips_2022}
\usepackage[natbib=true, style=authoryear, uniquename=false]{biblatex}
\renewbibmacro{in:}{}
\bibliography{main.bib}
\usepackage{utils/my_commands}
\usepackage{subcaption}
\usepackage{caption}
\usepackage{enumitem}
\setlist[itemize]{leftmargin=5.5mm} 
\setlist[enumerate]{leftmargin=5.5mm} 
\usepackage{makecell}

\usepackage{subfiles} 
\usepackage[utf8]{inputenc} 
\usepackage[T1]{fontenc}    
\usepackage[bookmarksnumbered, colorlinks=true, citecolor=blue, linkcolor=blue]{hyperref}
\usepackage{url}            
\usepackage{booktabs}       
\usepackage{amsfonts}       
\usepackage{nicefrac}       
\usepackage{microtype}      
\usepackage{xcolor}         

\title{A Consistent Estimator for Confounding Strength}

\author{%
}

\begin{document}

\maketitle
\textbf{Luca Rendsburg$^{\boldsymbol{*}}$}
\hfill \textsc{luca.rendsburg@uni-tuebingen.de} \\
\textit{Department of Computer Science}\\
\textit{University of T\"{u}bingen}\medskip\\
\textbf{Leena Chennuru Vankadara$^{\boldsymbol{*}}$}
\hfill \textsc{leena.chennuru-vankadara@uni-tuebingen.de} \\
\textit{Department of Computer Science}\\
\textit{University of T\"{u}bingen}\medskip\\
\textbf{Debarghya Ghoshdastidar}
\hfill \textsc{ghoshdas@cit.tum.de} \\
\textit{School of Computation Information and Technology}\\
\textit{Technical University of Munich}\\
\textit{Munich Data Science Institute}\medskip\\
\textbf{Ulrike von Luxburg}
\hfill \textsc{ulrike.luxburg@uni-tuebingen.de} \\
\textit{Department of Computer Science}\\
\textit{University of T\"{u}bingen}

\def\thefootnote{$\boldsymbol{*}$}\footnotetext{Equal contribution.}
\def\thefootnote{\arabic{footnote}}

\title{\huge A Consistent Estimator for Confounding Strength}

\begin{abstract}
Regression on observational data can fail to capture a causal relationship in the presence of unobserved confounding. 
Confounding strength measures this mismatch, but estimating it requires itself additional assumptions. A common assumption is the independence of causal mechanisms, which relies on concentration phenomena in high dimensions. While high dimensions enable the estimation of confounding strength, they also necessitate adapted estimators. In this paper, we derive the asymptotic behavior of the confounding strength estimator by \citet{Jan:2018} and show that it is generally not consistent. We then use tools from random matrix theory to derive an adapted, consistent estimator.
\end{abstract}

\section{Introduction}\label{sec:intro}

A common machine learning task is to learn the influence of features $x$ on a target variable $y$ from a set of observations $\{(x_i,y_i)\}_{i=1}^n$.
In many applications, we are not only interested in the statistical problem of predicting $y$ after \textit{observing} $x$; instead, we ask the causal question of how $y$ changes after \textit{intervening} on $x$.
Unfortunately, the causal dependence structure between $x$ and $y$ is in general not identifiable from their statistical dependencies \citep{Pea:2009book}. 
Simply regressing $y$ on $x$ attributes all dependencies to direct causal influence and is therefore only appropriate when $x$ causes $y$ without hidden confounders.
However, this solution can be grossly misleading in the other possible cases where $y$ causes $x$ or both are caused by a common confounder \citep{Rei:1956}.

For example, assume we want to predict how increasing a person's education $x$ affects their income $y$.
It could be that a higher education is a requirement for well-paying jobs (education causes income), in which case increasing the education directly increases the income. However, even if we rule out the possibility that income causes education, education and income could both be affected by some hidden confounders such as the socioeconomic status of the parents. A priori, it is unclear to what extent the observed statistical dependence between $x$ and $y$ is due to direct causal influence or due to such confounding factors.

This fundamental non-identifiability issue of causal from observational structure can be addressed in different ways. 
One way is access to additional data such as data from different environments \citep{Pet:2016, Hei:2018} or instrumental variables \citep{Bow:1990, Imb:1994}, which reduces the causal learning problem to a statistical one.
Alternatively, one can assume that the underlying causal model follows a certain data-generating process such as additive noise models \citep{Kan:2003, Hoy:2008Nonlin, Zha:2009}. 
This reduces the number of causal models which can explain a given observational structure and therefore mitigates the non-identifiability. 
A more abstract approach to choose a causal model among those compatible with an observational structure is to postulate certain information-theoretic properties of the causal model. 
For example, the causal directions are those that maximize conditional entropies or the causal factorization of the joint distribution is the one with minimal Kolmogorov complexity \citep{Sun:2006, Jan:2010AlgMarkov, Blo:2018, Mar:2019}.

In this paper, we theoretically analyze the confounding strength estimator by \citet{Jan:2018}.
This estimator assumes that $x$ causes $y$ and aims to estimate the strength of unobserved confounding based on observational data $\{(x_i,y_i)\}_{i=1}^n$. 
Here, the confounding strength is defined as the discrepancy between the causal effect of $x$ on $y$ and the statistical regression vector.
To mitigate the non-identifiability, the estimator considers a linear Gaussian causal model under the assumption of independent causal mechanisms, a common assumption in causal learning \citep{Jan:2010AlgMarkov, Lem:2013, Pet:2017}. Abstractly, this principle states that the different causal mechanisms share no information. While the task of confounding strength estimation remains ill-posed in finite dimensions, it becomes solvable in the high-dimensional limit due to concentration of measure phenomena.
Crucially, this approach therefore requires large dimension $d$ to reduce the non-identifiability error, but at the same time requires an even larger number of samples $n\gg d$ to reduce the finite-sample error. This is because it uses the empirical covariance matrix and regression vector in an intermediate step to estimate the corresponding population quantities, which is only consistent for $n\gg d$. 
It is therefore not guaranteed that this estimator is consistent in the high-dimensional regime.
We address this issue by analyzing this estimator, from here on referred to as the plug-in estimator, in the proportional asymptotic regime $n, d\to \infty$ with $d/n\to\gamma\in [0,1)$ and make the following contributions:
\begin{itemize}
    \item We derive the asymptotic behavior of the plug-in estimator for confounding strength from \citet{Jan:2018} in the proportional asymptotic regime and show that it is not generally consistent. We also show that the approach based on population instead of finite-sample quantities is consistent.
    \item We derive a consistent estimator for confounding strength by correcting the above estimator with tools from random matrix theory. 
    \item We demonstrate the improvement experimentally on finite-dimensional data from our causal model.
\end{itemize}

The paper is structured as follows.
Section~\ref{sec:related_work} gives an overview of related work on causal inference under unobserved confounding.
Section~\ref{sec:preliminaries} introduces the confounded causal model, the measure of confounding strength, and basic notions from random matrix theory which are needed for the analysis.
Section~\ref{sec:population_and_plugin} describes the general approach of \citet{Jan:2018} and shows that it is consistent based on population quantities in Section~\ref{sec:population}, but generally biased based on plug-in quantities in Section~\ref{sec:plugin}.
A corrected, consistent estimator for confounding strength is then derived in Section~\ref{sec:consistent_estimator}.
Section~\ref{sec:discussion} concludes with a discussion.

\section{Related work}\label{sec:related_work}

Learning causal relationships under the presence of unobserved confounding has been investigated by multiple works.
\citet{Hoy:2008Lin} detect the causal direction in linear non-Gaussian models based on the structure of the mixing matrix and \citet{Jan:2009} do so for non-linear additive noise models.
\citet{Jan:2011} detect low-complexity confounding based on a purity criterion for conditional distributions.
\citet{Kal:2019} decide whether a causal model is confounded based on the algorithmic Markov condition.
\citet{Che:2022} consider the stability of the regression vectors under different environments as an indication for causal influence.

Our paper falls into another line of work that detects confounding based on the assumption of independent causal mechanisms. This assumption induces certain non-generic alignments between the coefficients of the observational distribution, which can be used to identify confounding.
\citet{Bel:2021} use this assumption to learn a sparse causal DAG under dense confounding.
\citet{Jan:2017} introduce the notion of confounding strength and estimate it under scalar confounding. Their method is based on the observation that a weighted spectral measure of the covariance matrix concentrates in high dimensions.
\citet{Liu:2018} build on this idea by moving from the spectral measure to its first moment.
\citet{Jan:2018} extend this setting to multivariate confounding, which is the setting of our work.
\citet{Jan:2019} considers a subsequent task of learning a causal model with ridge regression. It uses an estimate of confounding strength to choose an appropriate regularization parameter, which is motivated by an analogy between finite sample error and confounding.
\citet{Van:2022} generalize the notion of confounding strength beyond independent causal mechanisms and characterize the relationship between confounding strength and the causal risk of ridge regression in the high-dimensional limit.

Another related field is sensitivity analysis for treatment-effect studies based on observational data. Sensitivity analysis aims to quantify how sensitive causal conclusions are to potential unobserved confounding \citep{Cor:2009}. Since this task suffers from the same non-identifiability issue as described above, early work relies on assumptions about the unobserved confounder \citep{Fla:1990, Van:2011}. A more recent, popular approach 
without assumptions gives bounds based on two (unknown) sensitivity parameters for how strong confounding would need to be in order to explain away any observed statistical associations between treatment and effect \citep{Din:2016, Sjo:2020, Pen:2022}.
The region of sensitivity parameters that explain away associations can be condensed into a single E-value, which acts as a measure of confounding strength and can be computed from observational data \citep{Van:2017, Van:2019}.

\section{Preliminaries}\label{sec:preliminaries}
This preliminary section introduces our confounded causal and a notion of confounding strength in Section~\ref{sec:causal_model}, as well as basic tools from random matrix theory needed for analysis in Section~\ref{sec:rmt_basics}.

\subsection{The confounded causal model}\label{sec:causal_model}
\begin{figure*}[t]
\centering
\begin{tikzpicture}[x=1.5cm,y=1.0cm]
    \node[obs] (x) {$x$} ; %
    \node[latent, above right=of x] (z) {$z$} ; %
    \node[obs, below right=of z] (y) {$y$} ; %
    \node[latent, above=of y] (eps) {$\varepsilon$} ; %

    \edge {x} {y}  ; %
    \edge {eps} {y} ; %
    \edge {z} {y} ; %
    \edge {z} {x} ; %
    
    \path (x) -- (y) node [midway,above](TextNode){$\beta$};
    \path (z) -- (x) node [midway,above](TextNode){$M$};
    \path (z) -- (y) node [midway,above](TextNode){$\alpha$};
\end{tikzpicture}%
\hspace{50pt}
\begin{tikzpicture}[x=1.5cm,y=1.0cm]
    \node[obs] (x) {$x$} ; %
    \node[obs, right=of x] (y) {$y$} ; %
    \node[latent, above=of y] (eps) {$\tilde{\varepsilon}$} ; %

    \edge {x} {y} ; %
    \edge {eps} {y} ; %
    
    \path (x) -- (y) node [midway,above](TextNode){$\tilde{\beta}$};
\end{tikzpicture}
\caption{\textbf{Left:} DAG corresponding to the causal model \eqref{eq:causal_model}. \textbf{Right:} corresponding observational model as in Proposition~\ref{lem:observational_dist} with $\tilde{\varepsilon}\sim\mathcal{N}(0,\sigmastat)$. Unobserved variables are dashed.}
\label{fig:causal_DAG}
\end{figure*}
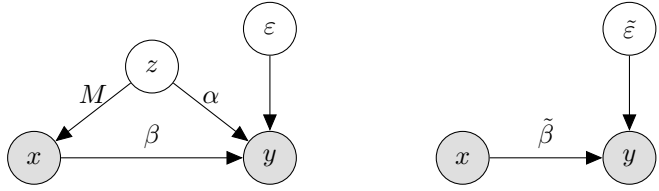

We first describe the problem setup and introduce basic quantities.
We consider a confounded causal model with linear conditionals and Gaussian distributions. Specifically, we define the causal model in terms of its structural equations
\begin{equation}\label{eq:causal_model}
\begin{aligned}
    z&\sim\Gauss{0}{I_l}\,,\\
    \varepsilon &\sim \Gauss{0}{\sigmacaus}\,,\\
    x&= Mz\,,\\
    y&= x^T\beta + z^T\alpha + \varepsilon\,.
\end{aligned}
\end{equation}
Figure~\ref{fig:causal_DAG} shows the corresponding directed acyclic graph (DAG).
The model depends on a set of hyperparameters $\alpha\in\R^l, \beta\in\R^d, M\in\R^{d\times l}$ with $l\geq d$ and the noise $\sigmacaus\geq 0$. 
All variables $x,y,z$ have mean 0 and the covariance of the features is given by $\Sigma\coloneqq\Cov(x)=MM^T\in\R^{d\times d}$. We additionally assume that $M$ has full rank $d$ such that $\Sigma$ is invertible.
We use the notation $\norm{x}^2_\Sigma\coloneqq x^T\Sigma x$ for the generalized norm, $M^+$ for the pseudo-inverse of $M$, and $M^{+T}\coloneqq (M^+)^T$ as shorthand.

By construction, $\beta$ describes the causal influence of $x$ on $y$. This is formally captured by the interventional distribution of the \textit{do}-calculus \citep{Pea:2009} under which $y=x_0^T\beta + z^T\alpha + \varepsilon$ is only a random variable in $z,\varepsilon$ and therefore $\E_{y|do(x=x_0)}y=x_0^T\beta$.
However, we do not assume access to interventional data; instead, we only observe values values $(x,y)$. The corresponding statistical dependencies between $x$ and $y$ are captured by the usual conditional distribution:

\begin{lemma}[Observational distribution]\label{lem:observational_dist}
For the causal model~\eqref{eq:causal_model}, the observational distribution of $y$ given $x$ is $y|x\sim\mathcal{N}(x^T\betaStat, \sigmastat)$, where $\betaStat=\beta+M^{+T}\alpha$ and $\sigmastat=\sigmacaus+\norm{\alpha}^2_{I_l-M^+M}$.
\end{lemma}
\begin{proof}
Since $z\sim\Gauss{0}{I_l}$ is Gaussian and $x=Mz$ is a linear map, it is a standard result that $z^T|x$ is Gaussian again with parameters $z^T|x\sim\Gauss{x^TM^{+T}}{I-M^+M}$. Subsequently, we have $z^T\alpha|x\sim\Gauss{x^TM^{+T}\alpha}{\norm{\alpha}^2_{I-M^+M}}$. With $y=x^T\beta + z^T\alpha+\varepsilon$, we arrive at
\begin{align*}
    y|x\sim\Gauss{x^T(\beta+M^{+T}\alpha)}{\sigmacaus+\norm{\alpha}^2_{I-M^+M}}=\Gauss{x^T\betaStat}{\sigmastat}\,.
\end{align*}
\end{proof}

The statistical parameter $\betaStat$ can also be viewed as the result of regressing $y$ on $x$ on the population level, that is, $\betaStat=\Cov(x)^+\Cov(x,y)$.
Notice that $\betaStat$ is equal to the causal parameter $\beta$ up to an error term $M^{+T}\alpha$, which results from the influence of the confounder $z$ on $y$.
This error term cannot be identified even if we have access to the full joint distribution $\proba_{(x,y)}$, which demonstrates the fundamental non-identifiability issue of causal learning.
To quantify the error of incorrectly treating $\betaStat$ as the causal parameter, \citet{Jan:2017} propose the following measure of confounding strength:

\begin{definition}[Measure of confounding strength, \citep{Jan:2017}]
The \emph{confounding strength} $\conf$ for the causal model~\eqref{eq:causal_model} is defined as the relative error between statistical parameter $\betaStat$ and causal parameter $\beta$ via
\begin{align}\label{eq:def_conf}
    \conf \coloneqq \frac{\norm{\tilde{\beta}-\beta}^2}{\norm{\beta}^2 + \norm{\tilde{\beta}-\beta}^2}\,.
\end{align}
\end{definition}
The confounding strength $\conf$ takes values in $[0,1]$, where $\conf=0$ describes the unconfounded case $\alpha=0$ for which $\betaStat=\beta$ and $\conf=1$ describes the purely confounded case $\beta=0$. A larger confounding strength implies that the statistical parameter is further away from the causal parameter.

The goal of this paper is to estimate the confounding strength based on finite samples $\{(x_i,y_i)\}_{i=1}^n\subset\R^d\times\R$ from the observational distribution $\proba_{(x,y)}$, which we compactly write as $X\in\R^{d\times n}$ and $Y\in\R^n$.
We define two quantities which are central to the following estimators, namely the sample covariance matrix $\empcov\coloneqq \frac{1}{n}XX^T$ and the result of regressing $Y$ on $X$, $\betaMinNorm\coloneqq (\frac{1}{n}XX^T)^+\frac{1}{n}XY$.

\subsection{Basic tools from random matrix theory}\label{sec:rmt_basics}
We briefly recap some standard tools and results from random matrix theory to analyze the following estimators for confounding strength in the high-dimensional regime. The analysis is based on the following two objects, which capture the spectrum of a matrix:
\begin{definition}[Empirical spectral distribution and Stieltjes transform]\label{def:esd}
Let $\Sigma\in\R^{d\times d}$ be a symmetric matrix with eigenvalues $\lambda_1,\dotsc,\lambda_d$. The  \emph{empirical spectral distribution} of $\Sigma$ is defined as the normalized counting measure of its eigenvalues $\mu_\Sigma\coloneqq \frac{1}{d}\sum_{i=1}^d\delta_{\lambda_i}$. The corresponding \emph{Stieltjes transform} of this measure is defined as the function $m_{\Sigma}(z)\coloneqq\sum_{i=1}^d\frac{1}{\lambda_i-z}$ for $z\in\C\setminus\{\lambda_1,\dotsc,\lambda_d\}$.
\end{definition}

We need to characterize the spectral behavior of the empirical covariance matrix $\empcov=\frac{1}{n}XX^T\in\R^{d\times d}$ and the closely related empirical kernel matrix $\empker=\frac{1}{n}X^TX\in\R^{n\times n}$. The following standard result relates their limiting spectra to the spectrum of the population covariance in terms of Stieltjes transforms:
\begin{theorem}[Asymptotics of the sample covariance matrix, \citep{Sil:1995}]\label{thm:sample_covariance_matrix}

Let $n,d\to\infty$ such that $d/n\to\gamma\in(0,\infty)$ and assume that the empirical spectral distribution of the covariance $\Sigma$ converges, that is, $\esdCov\aslim\esdCovLim$ with bounded support and corresponding Stieltjes transform $\stieltjesCovLim$. Then it holds that $\esdEmpcov\aslim\esdEmpcovLim$ and $\esdEmpker\aslim\esdEmpkerLim$ as $d\to\infty$, where $\esdEmpcovLim,\esdEmpkerLim$ are the unique measures having Stieltjes transforms $\stieltjesEmpcovLim(z)$ and $\stieltjesEmpkerLim(z)$, respectively. For $z\in\C\setminus\R_+$, they satisfy
\begin{align}
    \stieltjesEmpcovLim(z)&=\frac{1}{\gamma}\stieltjesEmpkerLim(z)+\frac{1-\gamma}{\gamma z}\,,\label{eq:stieltjes_cov_ker}\\
    \stieltjesCovLim\left(-\frac{1}{\stieltjesEmpkerLim(z)}\right)&=-z\stieltjesEmpcovLim(z)\stieltjesEmpkerLim(z)\,.\label{eq:limiting_stieltjes}
\end{align}

\end{theorem}
A corresponding version of Eq.~\eqref{eq:stieltjes_cov_ker} holds in finite dimensions and simply reflects the fact that $\empcov$ and $\empker$ share the same eigenvalues up to the eigenvalue 0 with multiplicity $\abs{n-d}$.
Eq.~\eqref{eq:limiting_stieltjes} is the main result that connects the limiting Stieltjes transforms of the empirical matrices $\empcov$ and $\empker$ to the limiting Stieltjes transform of the population covariance $\Sigma$.
The solution $\stieltjesEmpcovLim$ to this equation remains implicitly defined in all but the simplest case $\Sigma=I_d$, where $\stieltjesEmpcovLim$ is the Stieltjes transform of a Mar\u{c}enko-Pastur distribution.

\section{Asymptotic behavior of the population and plug-in estimators for confounding strength}\label{sec:population_and_plugin}
In this section, we describe the general approach for estimating confounding strength based on the assumption of independent causal mechanisms \citep{Jan:2018}.
We show that the estimator is consistent based on population quantities in Section~\ref{sec:population}, but is generally biased for $n\not\gg d$ based on sample (plug-in) quantities in Section~\ref{sec:plugin}.

The main ingredient to tackle the non-identifiability of the causal model is the assumption of independent causal mechanisms, a common assumption in causal learning \citep{Jan:2010AlgMarkov}. This abstract principle states that the physical mechanisms of a causal model that transfers causes to effect share no information. A possible translation for the causal model~\eqref{eq:causal_model} is the assumption that the mechanisms $\alpha$ and $\beta$ are drawn from independent rotationally invariant distributions. Specifically, we assume that $\alpha$ and $\beta$ are independent with  $\alpha\sim\mathcal{N}(0,\sigalpha I_l)$ and $\beta\sim\mathcal{N}(0,\sigbeta I_d)$ for unknown hyperparameters $\sigalpha,\sigbeta \geq 0$.
Intuitively, this assumption facilitates estimation because it implies a certain alignment between the covariance matrix $\Sigma=MM^T$ and the regression vector $\betaStat=\beta + M^{+T}\alpha$: for large confounding $\alpha$, the error term $M^{+T}\alpha$ is aligned with small singular value directions of $M$. Correspondingly, $\betaStat$ is aligned with small eigendirections of $\Sigma$.

\begin{assumption}\label{assumptions}~\vspace{-10pt}
We make the following assumptions about the (sequence) of causal models:
\begin{enumerate}[label=(A\arabic*),leftmargin=\widthof{[(A4)]}+\labelsep]
    \item\label{ass:ICM} The parameters $\alpha,\beta$ of model~\eqref{eq:causal_model} are independently sampled with $\alpha\sim\mathcal{N}(0,\sigalpha I_l)$ and $\beta\sim\mathcal{N}(0,\sigbeta I_d)$ for hyperparameters $\sigalpha,\sigbeta \geq 0$.
    \item\label{ass:dimensions} The number of samples $n$, data dimension $d$, and latent confounder dimension $l$ are in the proportional asymptotic regime, that is, $n,d,l\to\infty$ such that $d/n\to \gamma\in (0,1)$ and $l/d\to\tilde{\gamma}\geq 1$.
    \item\label{ass:popcov} The empirical spectral distribution $\esdCov$ of the population covariance $\Sigma$ converges almost surely as $d\to\infty$ to a distribution $\nu$ with bounded support, that is, $\supp(\esdCovLim)\subseteq[h_1,h_2]$ with $0<h_1\leq h_2<\infty$.
\end{enumerate}
\end{assumption}
Assumption~\ref{ass:ICM} is the assumption of independent causal mechanisms.
Assumption~\ref{ass:dimensions} captures that this approach to confounding strength estimation requires high dimensions so that concentration effects can mitigate the non-identifiability issue. We exclude the case $\gamma\geq 1$, because there estimation of the term $\frac{1}{d}\Tr(\Sigma^{-1})$ (which later turns out to be relevant) is hard, see \citet[Remark 2.11]{Cou:2022} for a discussion.
The restriction on the latent dimensions $\tilde{\gamma}\geq 1$ ensures that $l\geq d$ so that the population covariance $\Sigma=MM^T$ with $M\in\R^{d\times l}$ can be full rank, which is necessary for Assumption~\ref{ass:popcov}. 

\begin{remark}
The assumption of independent causal mechanisms alone does not resolve the non-identifiability issue and it also does not enable estimation of the multivariate vectors $\alpha$ or $\beta$. However, \emph{scalar} functions of these parameters can concentrate in high dimensions. In particular, this happens for confounding strength.
\end{remark}
The following key lemma states that random quadratic forms can concentrate around their trace.
\begin{lemma}[Quadratic-form-close-to-the-trace, {\citep[Lemma B.26]{Bai:2010}}]\label{lem:QCT}
Let $x=(x^1,\dotsc,x^d)\in\R^d$ have independent entries $x^i$ of zero mean, unit variance and $\E[\abs{x^i}^K]\leq\nu_K$ for some $K\geq 1$. Then for $A\in\R^{d\times d}$ and $k\geq 1$,
\begin{align*}
    \E\left[\abs{x^TAx-\Tr A}^k\right]\leq C_k\left[\left(\nu_4\Tr\left(AA^T\right)\right)^{k/2}+\nu_{2k}\Tr\left(AA^T\right)^{k/2}\right]\,,
\end{align*}
for some constant $C_k>0$ independent of $d$. In particular, if the operator norm of $A$ satisfies $\norm{A}\leq 1$ and the entries of $x$ have bounded eigth-order moment,
\begin{align*}
    \E\left[\left(x^TAx-\Tr A\right)^4\right]\leq Cd^2\,,
\end{align*}
for some $C>0$ independent of $d$, and consequently
\begin{align*}
    \frac{1}{d}x^TAx-\frac{1}{d}\Tr A \xrightarrow[d\to\infty]{a.s.} 0\,.
\end{align*}
\end{lemma}
Using this lemma, we directly obtain concentration of the confounding strength.
\begin{corollary}[Confounding strength concentrates]\label{cor:conf_concentration}
Under Assumption~\ref{assumptions},
\begin{align}\label{eq:conf_concentration}
    \conf - \frac{\tauPop\cdot\thetaTrue}{1+\tauPop\cdot\thetaTrue} \aslim 0\,,
\end{align}
where $\tauPop\coloneqq \frac{1}{d}\Tr(\Sigma^{-1})$ and $\thetaTrue\coloneqq\sigalpha/\sigbeta$. 
\end{corollary}
\begin{proof}
By rewriting the confounding strength from Eq.~\eqref{eq:def_conf} in terms of the hyperparameters $\alpha,\beta, M$, we see that it consists only of quadratic forms. These can be controlled by Lemma~\ref{lem:QCT}, which yields
\begin{align*}
    \conf 
    =\frac{\frac{1}{d}\alpha^TM^+M^{+T}\alpha}{\frac{1}{d}\beta^TI_d\beta + \frac{1}{d}\alpha^TM^+M^{+T}\alpha}
    \asapprox
    \frac{\frac{1}{d}\Tr(M^+M^{+T})\sigalpha}{\frac{1}{d}\Tr(I_d)\sigbeta+\frac{1}{d}\Tr(M^+M^{+T})\sigalpha}
    \,=\frac{\tauPop\cdot\thetaTrue}{1+\tauPop\cdot\thetaTrue}\,.
\end{align*}
\end{proof}

It only remains to estimate the trace term $\tauPop$ and the ratio $\thetaTrue$.
In the following, we distinguish between three different kinds of estimators for various quantities: estimators based on the population quantities $\Sigma,\betaStat$, based on the plug-in quantities $\empcov, \betaMinNorm$, and consistent estimators derived by random matrix theory. For example, we write $\tauPop$, $\tauPlugin$, or $\tauRMT$. 

\subsection{The population estimator for confounding strength is consistent}\label{sec:population}
First, we consider estimation based on the population quantities $\Sigma$ and $\betaStat$, which basically assumes that there are no finite-sample issues. In this case, $\tauPop=\frac{1}{d}\Tr(\Sigma^{-1})$ is known and does not need to be estimated.
To estimate $\thetaTrue=\sigalpha/\sigbeta$ observe that Assumption~\ref{assumptions}\ref{ass:ICM} on $\alpha$ and $\beta$ implies $\betaStat=\beta+M^{+T}\alpha\sim\mathcal{N}(0,\sigbeta + \sigalpha\Sigma^{-1})$. 
With respect to the uniform distribution on the sphere $S^{d-1}$, the distribution of the normalized vector $\betaStat/\norm{\betaStat}$ has the log density $\log p_{\thetaTrue}(v)= -.5(\log\det (\Sigma + \thetaTrue)+d\log\scalprod{v}{\Sigma (\Sigma +\thetaTrue)^{-1}v} - \log \det \Sigma)$, where $v\in S^{d-1}$. Correspondingly, $\thetaTrue$ can then be estimated via maximum likelihood estimation as\footnote{Maximum likelihood estimation on the density of $\betaStat$ directly leads to the same optimality condition for $\thetaPop$.}
\begin{align}\label{eq:logprob_pop}
    \thetaPop= \argmin_{\theta\geq 0} \logprobPop(\theta),\quad \text{where}\quad
    \logprobPop(\theta)=\frac{1}{d}\log\det (\Sigma + \theta)+\log\scalprod{\frac{\betaStat}{\norm{\betaStat}}}{\Sigma (\Sigma + \theta)^{-1}\frac{\betaStat}{\norm{\betaStat}}}\,.
\end{align}

In summary, we consider the following population estimator for confounding strength.
\begin{definition}[Population estimator for confounding strength]\label{def:population_estimator}
Given $\Sigma$ and $\betaStat$, the \emph{population estimator} for confounding strength $\confPop$ is defined as
\begin{align}
    \confPop=\frac{\tauPop\cdot\thetaPop}{1+\tauPop\cdot\thetaPop}\,,\label{eq:def_population_estimator}
\end{align}
where $\tauPop= \frac{1}{d}\Tr(\Sigma^{-1})$ and $\thetaPop$ is given by Eq.~\eqref{eq:logprob_pop}.
\end{definition}

We now analyze this estimator by analyzing the asymptotic behavior of $\thetaPop$ from Eq.~\eqref{eq:logprob_pop}. Since $\thetaPop$ is implicitly defined as the minimizer of the function $\logprobPop$, we first derive its asymptotic behavior as an intermediate step. Specifically, we consider its derivative, which is given by
\begin{align}\label{eq:logprob_pop_derivative}
    \partial_\theta\logprobPop(\theta)=\stieltjesCov(-\theta) - \frac{\scalprod{\frac{\betaStat}{\norm{\betaStat}}}{\Sigma(\Sigma+\theta)^{-2}\frac{\betaStat}{\norm{\betaStat}}}}{\scalprod{\frac{\betaStat}{\norm{\betaStat}}}{\Sigma(\Sigma+\theta)^{-1}\frac{\betaStat}{\norm{\betaStat}}}}\,.
\end{align}
This idea is realized in the next theorem, which shows that the confounding strength estimator based on population quantities is consistent as $n,d\to\infty, d/n\to\gamma\in(0,1)$.

\begin{restatable}[Population estimator is consistent]{theorem}{consistencyPop}\label{thm:consistency_pop}
Under Assumption~\ref{assumptions} with $\thetaTrue > 0$,

\begin{enumerate}
 \item For every $\theta \geq 0$, the derivative of the function from Eqs.~\eqref{eq:logprob_pop} satisfies
    \begin{align}
    \partial_\theta\logprobPop_d(\theta)&\aslim (\theta-\thetaTrue)\Var_{\lambda\sim\esdCovLim}\left[\frac{1}{\lambda+\theta}\right]\E_{\lambda\sim\esdCovLim}\left[\frac{\lambda+\thetaTrue}{\lambda+\theta}\right]^{-1}\,,\label{eq:lim_derivative_pop}
\end{align}
\item For some $C>\thetaTrue$ and every $d\in\N$, let $\thetaPop_d$ be a root of $\partial_\theta\logprobPop_d$ in $[0, C]$ if it exists or $0$ otherwise. Additionally, assume that $\esdCovLim$ is not degenerate. Then the sequence $\{\thetaPop_d\}$ converges to $\thetaTrue$ almost surely.
\end{enumerate}
\end{restatable}
\begin{proof}
We just present a proof sketch here, the full proof is deferred to Appendix~\ref{app:population}.
For the first statement about the population function $\partial_\theta\logprobPop_d$ we treat the three terms in Eq.~\eqref{eq:logprob_pop_derivative} separately. The first term $\stieltjesCov(-\theta)$ converges to $\stieltjesCovLim(-\theta)$ by Assumption~\ref{assumptions}\ref{ass:popcov}. The two quadratic forms are handled by Lemma~\ref{lem:QCT} after rewriting $\betaStat=\beta+M^{+T}\alpha=\pvector{\sigma_\alpha M^{+T} & \sigma_\beta I_d}u$ for some $u\sim\mathcal{N}(0, I_{l+d})$. Plugging everything together and simplifying yields the result.

We prove the second statement by first upgrading the convergence of Eq.~\eqref{eq:lim_derivative_pop} to uniform convergence on $[0,C]$ using Vitali's convergence theorem \citep{titchmarsh1939theory}, and then conclude that the roots converge to the unique root $\thetaTrue$ of the limiting function using Hurwitz's theorem \citep{titchmarsh1939theory}.
\end{proof}

This theorem shows that the approach of minimizing the log probability based on population quantities in Eq.~\eqref{eq:logprob_pop} correctly estimates $\thetaTrue$ in the limit. Therefore, Eq.~\eqref{eq:def_population_estimator} leads to a consistent estimator for confounding strength.
For the second statement, it is necessary to assume that the limiting spectral distribution $\esdCovLim$ of $\Sigma$ is not degenerate, because otherwise $\Var_{\lambda\sim\esdCovLim}\left[1/(\lambda+\theta)\right]=0$. In this case, Eq.~\eqref{eq:lim_derivative_pop} states that the derivative $\partial_\theta\logprobPop$ converges to the constant 0 function, which contains no information about $\thetaTrue$.
This is perfectly in line with the intuition presented for this approach: estimation of confounding strength is made possible by an alignment of $\betaStat$ with small eigendirections of $\Sigma$, but if $\Sigma$ is a multiple of the identity (or, equivalently, the distribution of eigenvalues $\esdCovLim$ is degenerate), there is no particular small eigendirection. 

\subsection{The plug-in estimator for confounding strength is generally biased}\label{sec:plugin}
The population estimator considered above crucially relies on the population quantities $\Sigma$ and $\betaStat$, which are not directly available. In practice, we only have access to the corresponding empirical quantities $\empcov$ and $\betaMinNorm$ based on samples $X,Y$. This section considers the resulting plug-in estimator for confounding strength as introduced by \citet{Jan:2018} and shows in a similar asymptotic analysis that this estimator is generally biased.
Formally, the plug-in estimator follows the same structure as Definition~\ref{def:population_estimator}, but replaces the population quantities $\Sigma,\betaStat$ with the empirical quantities $\empcov, \betaMinNorm$.
\begin{definition}[Plug-in estimator for confounding strength, \citep{Jan:2018}]\label{def:plugin_estimator}
The \emph{plug-in estimator} for confounding strength $\confPlugin$ is defined as
\begin{align}
    \confPlugin = \frac{\tauPlugin\cdot\thetaPlugin}{1+\tauPlugin\cdot\thetaPlugin}\,,\label{eq:def_plugin_estimator}
\end{align}
where $\tauPlugin =\frac{1}{d}\Tr(\empcov^{-1})$ and $\thetaPlugin$ is given by 
\begin{align}\label{eq:logprob_plugin}
    \thetaPlugin= \argmin_{\theta\geq 0} \logprobPlugin(\theta),\quad \text{where}\quad
    \logprobPlugin(\theta)=\frac{1}{d}\log\det (\empcov + \theta)+\log\scalprod{\frac{\betaMinNorm}{\norm{\betaMinNorm}}}{(\empcov(\empcov + \theta)^{-1})\frac{\betaMinNorm}{\norm{\betaMinNorm}}}\,.
\end{align}
\end{definition}

The main issue with the plug-in estimator in the proportional asymptotic regime is that $\empcov$ and $\betaMinNorm$ are not consistent estimators for $\Sigma$ and $\betaStat$. Any subsequent estimators are therefore also not guaranteed to be consistent.
The first example of such behavior is given by the plug-in estimator $\tauPlugin= \frac{1}{d}\Tr(\empcov^{-1})$ for $\tauPop=\frac{1}{d}\Tr(\Sigma^{-1})$, one of the two quantities which need to be estimated in Eq.~\eqref{eq:conf_concentration}.

\begin{proposition}[Asymptotic trace of inverse covariance]\label{prop:trace_inv}
Under Assumption~\ref{assumptions}, it holds
\begin{align*}
    \tauPlugin - (1-\gamma)^{-1}\tauPop \xrightarrow[d\to\infty]{a.s.}0\,.
\end{align*}
\end{proposition}
\begin{proof}
In terms of Stieltjes transforms, the statement reads $(1-\gamma)\stieltjesEmpcov(0)-\stieltjesCov(0)\xrightarrow[d\to\infty]{a.s.}0$.
The limiting empirical and population Stieltjes transforms are given by $\stieltjesEmpcov(z)\aslim\stieltjesEmpcovLim(z)$ and $\stieltjesCov(z)\aslim\stieltjesCovLim(z)$ as $d\to\infty$, so it remains to relate $\stieltjesEmpcovLim(0)$ to $\stieltjesCovLim(0)$. By combining equations \eqref{eq:stieltjes_cov_ker} and \eqref{eq:limiting_stieltjes} from Theorem~\ref{thm:sample_covariance_matrix}, we get 
\begin{align*}
    \stieltjesCovLim\left(-\frac{1}{\stieltjesEmpkerLim(z)}\right) = \left(1-\gamma-z\stieltjesEmpcovLim(z)\right)\stieltjesEmpcovLim(z)\,.
\end{align*}
Taking $z\to0$, it is $1/\stieltjesEmpkerLim(z)\to 0$ and therefore we get by continuity that $\stieltjesCovLim(0)=(1-\gamma)\stieltjesEmpcovLim(0)$.
\end{proof}

This result shows that the plug-in estimator for the trace of the inverse covariance matrix is off by a factor of $(1-\gamma)$. This factor is negligible in the case $n\gg d$ where $\gamma=d/n\approx 0$, but becomes increasingly relevant as $\gamma$ grows.


Next, we treat the plug-in estimator $\thetaPlugin$ similarly as $\thetaPop$ in Theorem~\ref{thm:consistency_pop} and show that it is generally biased. Here, $\partial_\theta\logprobPlugin$ is given analogously to Eq.~\eqref{eq:logprob_pop_derivative}.

\begin{restatable}[Plug-in estimator is generally biased]{theorem}{inconsistencyPlugin}\label{thm:inconsistency_plugin}
Under Assumption~\ref{assumptions} with $\thetaTrue > 0$,
\begin{enumerate}
 \item For all $\theta \geq 0$, the derivative of the function from Eq.~\eqref{eq:logprob_plugin} satisfies
\begin{align}
   \partial_\theta\logprobPlugin_d(\theta)&\aslim \left[\theta-(1+\gamma\tilde{\gamma})\thetaTrue+\gamma\thetaTrue(1-\theta m(-\theta))\left(1+\frac{M(-\theta)}{M(-\theta)-m(-\theta)^2}\right)\right]h(\theta)\,,\label{eq:lim_derivative_plugin} 
\end{align}
with $h(\theta)=(M(-\theta)-m(-\theta)^2)(1-\theta m(-\theta)+(1-2\gamma+\gamma\tilde{\gamma})\thetaTrue m(-\theta)+\gamma\theta\thetaTrue m(-\theta)^2)^{-1}$, 
where \\ $m(-\theta)=\E_{\lambda\sim\mu}\left[1/(\lambda+\theta)\right]$, and $M(-\theta)=\E_{\lambda\sim\mu}\left[1/(\lambda+\theta)^2\right]$.
\item For every $d\in\N$, let $\thetaPlugin_d$ be a root of $\partial_\theta\logprobPlugin_d$ if it exists or $0$ otherwise.
Additionally, assume that $\tilde{\gamma}$ does not satisfy
\begin{align}
    \tilde{\gamma}= (1-\thetaTrue m(-\thetaTrue))\left(1+\frac{M(-\thetaTrue)}{M(-\thetaTrue)-m(-\thetaTrue)^2}\right)\,.\label{eq:plugin_consistency_cond}
\end{align}
Then the sequence $\{\thetaPlugin_d\}$ almost surely does not converge to $\thetaTrue$.
\end{enumerate}
\end{restatable}
\begin{proof}
We again only sketch the proof here, the full proof is deferred to Appendix~\ref{app:plugin}.
The proof for the first statement follows the same strategy as in Theorem~\ref{thm:consistency_pop}, but now deals with the sample quantities $\empcov,\betaMinNorm$ in place of the population quantities $\Sigma, \betaStat$. 
Similarly as for $\betaStat$, we treat $\betaMinNorm$ by combining the equations $\betaMinNorm=(XX^T)^+XY$, $Y=X^T\betaStat+E$ for $E\sim\mathcal{N}(0,\sigmastat I_n)$, and $\betaStat=\beta+M^{+T}\alpha$ to obtain $\betaMinNorm = \pvector{\sigma_\alpha M^{+T} & \sigma_\beta I_d & \tilde{\sigma}(XX^T)^+X}v\quad\text{for some }v\sim\mathcal{N}(0, I_{l+d+n})$.
Additional complications arise because $\betaMinNorm$ depends on both the population term $M$ and the empirical quantities. This produces mixed terms $\Tr[(\empcov+\theta)^{-1}\empcov\Sigma^+]$ for $k\in\{1,2\}$, which need to be treated with a separate result by \citet{Led:2011} in Lemma~\ref{lem:mixed_traces}.

For the second statement, we use similar arguments as in the proof of Theorem~\ref{thm:consistency_pop} to show that the convergence $\thetaPlugin_d\to\thetaTrue$ implies that $\thetaTrue$ is a root of the right hand side in Eq.~\eqref{eq:lim_derivative_plugin}. This is equivalent to Eq.~\eqref{eq:plugin_consistency_cond}, which does not hold by assumption.
\end{proof}

The limiting derivative for the plug-in estimator in Eq.~\eqref{eq:lim_derivative_plugin} is phrased in terms of the limiting sample distribution $\esdEmpcovLim$ instead of the limiting population distribution $\esdCovLim$. The main structural difference to Eq.~\eqref{eq:lim_derivative_pop} is the existence of an additional term $\gamma\thetaTrue(1-\theta m(-\theta))(1+M(-\theta)/(M(-\theta)-m(-\theta)^2))$, which prevents a closed-form expression for the corresponding roots $\thetaPlugin$ of this function. We therefore cannot directly exclude the possibility that $\thetaTrue$ is a root, in which case the plug-in estimator would be consistent.
However, by simply plugging in $\thetaTrue$ in the limiting derivative, we see that $\thetaTrue$ being a root is equivalent to the condition in Eq.~\eqref{eq:plugin_consistency_cond}. This condition generally does not hold, because the limiting ratio of dimensions $\tilde{\gamma}=\lim_{d,l\to\infty}l/d$ on the left hand side stands in no special relationship to the terms on the right hand side. Therefore, the plug-in estimator $\thetaPlugin$ is generally a biased estimator for $\thetaTrue$. This means that the resulting plug-in estimator for confounding strength $\confPlugin$ is generally a biased estimator for the true confounding strength $\conf$.

\section{A consistent estimator for confounding strength}\label{sec:consistent_estimator}
In this section, we derive a novel estimator for confounding strength using tools from random matrix theory. We show that this estimator consistently recovers the true confounding strength in the high-dimensional asymptotic limit ($n, d \rightarrow \infty, d / n \rightarrow \gamma \in (0, 1)$). To this end, we can derive a consistent estimator of $\thetaRMT$ by first consistently estimating $\logprobPop(\theta)$ and then finding the minimizer of this function. While this procedure indeed yields a consistent estimator, it is stochastic, which can adversely affect the optimization algorithm at finite $d$. Therefore, we also provide a consistent estimator based on finding the zeros of $\partial_{\theta}\logprobPop(\theta)$ which is deterministic given a fixed sample.
Coupled with the consistent estimator for $\tauPop$ in Proposition~\ref{prop:trace_inv}, we arrive at a consistent estimator for confounding strength. 
 \subsection{A consistent estimator for $\logprobPop(\theta)$.}
Recall from Eq.~\eqref{eq:logprob_pop} that maximum likelihood estimation of $\thetaTrue$ is equivalent to the optimization problem
\begin{align*}
    \thetaPop\coloneqq \argmin_{\theta\geq 0} \logprobPop(\theta),\quad \text{where}\quad
    \logprobPop(\theta)=\frac{1}{d}\log\det (\Sigma + \theta)+\log \lscalprod{\frac{\betaStat}{\norm{\betaStat}}}{\Sigma (\Sigma + \theta)^{-1}\frac{\betaStat}{\norm{\betaStat}}}\,.
\end{align*}

 
 To consistently estimate $\logprobPop(\theta)$, it suffices to  consistently estimate the two quantities $\frac{1}{d}\log\det (\Sigma + \theta)$ and $\log \scalprod{\frac{\betaStat}{\norm{\betaStat}}}{\Sigma (\Sigma + \theta)^{-1}\frac{\betaStat}{\norm{\betaStat}}}$.
 We derive such estimators in Theorems \ref{thm:log_det_rmt_estim} and \ref{thm:consistent_Quadform} using tools from random matrix theory. The main results are included here and we defer the proofs to Appendix \ref{sec:rmt_proofs}.

\begin{theorem}[\textbf{A consistent estimator for log determinant, \citep{kammoun2011performance}}]
\label{thm:log_det_rmt_estim}
For any $\theta \in \R^+$, let $W = X + \sqrt{\theta} E$, where $E\in\R^{d\times n}$ is a random matrix with standard normal entries. Then, as $d, n \rightarrow \infty$ such that $d/n\to \gamma \in (0, 1)$,
\begin{equation*}
   \log \theta + \frac{1}{d} \log \det \frac{1}{n \theta}W W^T + (1 - \gamma) \log \frac{\gamma - 1}{\gamma} + 1 - \frac{1}{d} \log \det ( \Sigma + \theta) \aslim 0\,.
\end{equation*}
\end{theorem}
 
In other words, the function $g_1(\theta) = \log \theta + \frac{1}{d} \log \det \frac{1}{n \theta} W W^T + (1 - \gamma) \log ((\gamma - 1)/\gamma) + 1$ is a consistent estimator of $\log \det (\Sigma + \theta)$.
\vspace{3mm}

\begin{restatable}[A consistent estimator for the quadform]{proposition}{consistentQuadForm}\label{thm:consistent_Quadform}

Under Assumption~\ref{assumptions}, for any $\theta \in \R^+$, let $\eta$ be the unique solution in $\mathbb{R}^{-}$ satisfying $\tilde{m}(\eta) = 1/\theta$. Then, as $d, n \rightarrow \infty$ such that $d/n \rightarrow \gamma \in (0, 1)$,
\begin{equation*}
    \frac{ \frac{1}{d} \scalprod{\betaMinNorm}{\empcov (\empcov - \eta)^{-1} \betaMinNorm} -  \frac{S}{\theta} - \frac{S(1 - \gamma)}{\eta}}{{\frac{1}{d}\lVert\hat{\beta}\rVert^2-S\gamma m(0)}} -  \lscalprod{\frac{\betaStat}{\norm{\betaStat}}}{\Sigma (\Sigma + \theta)^{-1} \frac{\betaStat}{\norm{\betaStat}}} \aslim 0, 
\end{equation*}
where $S = (1 - \gamma)^{-1} \norm{Y}^2_{I - X^+X}/(nd)$. 
\end{restatable}

%
In other words, the function $g_2(\theta) = \log  \frac{ \frac{1}{d} \scalprod{\betaMinNorm}{\empcov (\empcov - \eta)^{-1} \betaMinNorm} -  \frac{S}{\theta} - \frac{S(1 - \gamma)}{\eta}}{{\frac{1}{d}\lVert\hat{\beta}\rVert^2-S\gamma m(0)}} $ is a consistent estimator of $\log \scalprod{\frac{\betaStat}{\norm{\betaStat}}}{\Sigma (\Sigma + \theta)^{-1} \frac{\betaStat}{\norm{\betaStat}}}.$ Thereby, for every $\theta \in \R^+$, as $n, d \rightarrow \infty$ as $d / n \rightarrow \gamma \in (0, 1)$,

\begin{equation}
    \label{eq:consistency_rmt_logprob}
     g_1(\theta) + g_2(\theta) - \logprobPop(\theta) \aslim 0
\end{equation}

In other words, a consistent estimator of $\logprobPop(\theta)$ is be given by $\logprobRMT(\theta) \coloneqq g_1(\theta) + g_2(\theta).$


\textbf{Stochasticity of the estimation.} Observe that the estimator for the log determinant given by $g_1(\theta)$ is not a deterministic function of a given sample $X, Y$ since the matrix $W$ is stochastic. Following arguments similar to the proof of Theorems \ref{thm:consistency_pop} and \ref{thm:inconsistency_plugin}\footnote{with an additional argument to deal with the stochasticity of the log det estimator.}, we can indeed obtain an asymptotically consistent estimator for confounding strength. 
However, at finite $d$ our experiments suggest that the stochasticity can adversely affect the optimization step. Furthermore, the dependence of $g_1(\theta)$ on $\theta$ is highly non-linear. Iterative optimization procedures require multiple evaluations (and therefore estimation of) $g_1(\theta)$ which considerably increases the computation complexity. To overcome these limitations, we also provide a deterministic and consistent estimator of $\theta$ by first consistently estimating the function $\partial_{\theta} \logprobPop(\theta)$ for any $\theta \in \R^+$ and showing that the roots of the estimating function asymptotically converges to $\theta^*$. 

\subsection{A consistent estiator for $\partial_{\theta} \logprobPop (\theta).$}

As derived in Eq.~\eqref{eq:logprob_pop_derivative}, the derivative of the lop probability function $\logprobPop (\theta)$ is given by

\begin{equation*}
    \partial_{\theta}\logprobPop(\theta)=  \frac{ \scalprod{\frac{\betaStat}{\norm{\betaStat}}}{\Sigma(\Sigma+\theta)^{-1}\frac{\betaStat}{\norm{\betaStat}}} \cdot \stieltjesCov(-\theta) - \scalprod{\frac{\betaStat}{\norm{\betaStat}}}{\Sigma(\Sigma+\theta)^{-2}\frac{\betaStat}{\norm{\betaStat}}}}{\scalprod{\frac{\betaStat}{\norm{\betaStat}}}{\Sigma(\Sigma+\theta)^{-1}\frac{\betaStat}{\norm{\betaStat}}}}\,.
\end{equation*}

In order to consistently estimate $\partial_{\theta}\logprobPop(\theta)$, it suffices to consistently estimate the three quantities $\stieltjesCov(-\theta)$, $\scalprod{\frac{\betaStat}{\norm{\betaStat}}}{\Sigma(\Sigma+\theta)^{-1}\frac{\betaStat}{\norm{\betaStat}}}$, and $\scalprod{\frac{\betaStat}{\norm{\betaStat}}}{\Sigma(\Sigma+\theta)^{-2}\frac{\betaStat}{\norm{\betaStat}}}$. Proposition \ref{thm:consistent_Quadform} provides us with a consistent estimator for the quantity $\scalprod{\frac{\betaStat}{\norm{\betaStat}}}{\Sigma(\Sigma+\theta)^{-1}\frac{\betaStat}{\norm{\betaStat}}}$. In Propositions \ref{thm:ConsistentStieltjes} and \ref{thm:consistent_DerivativeQuadform}, we derive estimators for the remaining quantities. 

\begin{proposition}[\textbf{Estimation of Stieltjes transform}]
\label{thm:ConsistentStieltjes}
Under the assumptions of Theorem~\ref{thm:sample_covariance_matrix}, for any $\theta \in \R^+$, let $\eta$ be the unique solution in $\R^-$ satisfying $\tilde{m}(\eta) = 1/\theta$. Then, as $d, n \rightarrow \infty$ such that $d / n \rightarrow \gamma \in (0, 1)$, 

\begin{equation*}
    - \frac{1}{\gamma \theta} \left(\frac{\eta}{\theta} - \gamma + 1\right) - \stieltjesCov(-\theta) \aslim 0\,.
\end{equation*}

\end{proposition}

\begin{proof}
From Theorem \ref{thm:sample_covariance_matrix}, we have that for any $z \in \mathbb{C}/\R^+$, $\stieltjesCovLim(-\frac{1}{\stieltjesEmpkerLim(z)}) = \left(1-\gamma-z\stieltjesEmpcovLim(z)\right)\stieltjesEmpcovLim(z).$ Letting $\eta \in \R^-$ such that $\tilde{m}(\eta) = 1/\theta,$ we arrive at the estimator.
\end{proof}

Now, we present a consistent estimator of the quadratic form $\scalprod{\frac{\betaStat}{\norm{\betaStat}}}{ \Sigma(\Sigma+\theta)^{-2} \frac{\betaStat}{\norm{\betaStat}}}.$ From Proposition \ref{thm:consistent_Quadform}, we know that for any $\theta \in \mathbb{R}^+$, $g_2(\theta)$ is a consistent estimator of $\scalprod{\frac{\betaStat}{\norm{\betaStat}}}{\Sigma(\Sigma+\theta)^{-1} \frac{\betaStat}{\norm{\betaStat}}}$. To derive an estimator of the quadratic form $\scalprod{\frac{\betaStat}{\norm{\betaStat}}}{\Sigma(\Sigma+\theta)^{-2} \frac{\betaStat}{\norm{\betaStat}}}$, we utilize the so-called derivative trick \citep{Dob:2018, Has:2019}. First observe that 

\begin{equation*}
\scalprod{\betaStat}{\Sigma(\Sigma+\theta)^{-2} \betaStat} = - \partial_{\theta} \left(\scalprod{\betaStat}{ \Sigma(\Sigma+\theta)^{-1} \betaStat}\right).
\end{equation*}
Furthermore, for every fixed $\theta \in \R^+$, we know that as $n, d \rightarrow \infty$ and $d / n \rightarrow \gamma \in (0, 1),$ 
$$g_2(\theta) - \lscalprod{\frac{\betaStat}{\norm{\betaStat}}}{ \Sigma(\Sigma+\theta)^{-1} \frac{\betaStat}{\norm{\betaStat}}} \aslim 0$$

It is also easy to verify that $g_2(\theta) - \scalprod{\frac{\betaStat}{\norm{\betaStat}}}{ \Sigma(\Sigma+\theta)^{-1} \frac{\betaStat}{\norm{\betaStat}}}$ is analytic and uniformly bounded in $\theta$ in the domain $\mathbb{R}^+$. Therefore, we can apply Vitali's convergence theorem to show that the limit of the derivatives converges to the derivative of the limit. Therefore a consistent estimator for the quadratic form $\scalprod{\frac{\betaStat}{\norm{\betaStat}}}{\Sigma(\Sigma+\theta)^{-2} \frac{\betaStat}{\norm{\betaStat}}}$ is given by $- \partial_{\theta}g_2(\theta)$ and is formally presented in Theorem \ref{thm:consistent_DerivativeQuadform}.

\begin{restatable}[Consistent estimator for quadratic form]{proposition}{consistentDerivativeQuadForm}\label{thm:consistent_DerivativeQuadform}

For any $\theta \in \R^+$, let $\eta$ be the unique solution in $\mathbb{R}^{-}$ satisfying $\tilde{m}(\eta) = 1/\theta$ and let $\eta' = 1/(\theta^2 \tilde{m}'(\eta)).$ As $d, n \rightarrow \infty$ such that $d/n \rightarrow \gamma \in (0, 1),$
\begin{equation*}
    \frac{ \frac{\eta'}{d} \scalprod{\betaMinNorm}{\empcov(\empcov+\theta)^{-2} \betaMinNorm} - \frac{S}{\theta^2} + \frac{S\eta'(1 - \gamma)}{\eta^2}}{{\frac{1}{d}\lVert\hat{\beta}\rVert^2-S\gamma m(0)}} -  \lscalprod{\frac{\betaStat}{\norm{\betaStat}}}{ \Sigma(\Sigma+\theta)^{-2} \frac{\betaStat}{\norm{\betaStat}}} \aslim 0, 
\end{equation*}
where $S = \frac{1}{(1 - \gamma) n d} \norm{Y}^2_{I - X^+X}/(nd)$. 
\end{restatable}


From Propositions \ref{thm:consistent_Quadform}, \ref{thm:ConsistentStieltjes}, and \ref{thm:consistent_DerivativeQuadform}, for any $\theta \in \R^+$, a consistent estimator of $\partial_{\theta} \logprobPop (\theta)$ is given by $$h_{\text{RMT}}(\theta) \coloneqq \frac{ \frac{ g_2(\theta)}{\gamma \theta} ( \gamma - 1 - \frac{\eta}{\theta})  - \partial_{\theta} g_2(\theta)}{g_2(\theta)}.$$

The RMT estimator for confounding strength is then naturally defined via the roots of $h_{\text{RMT}}(\theta)$ and RMT-corrected estimate of $\tauPop$ as is formally presented in Definition \ref{def:rmt_estimator} which consistently estimates the the true confounding strength $\zeta$.

\begin{definition}[RMT estimator for confounding strength]\label{def:rmt_estimator}
The \emph{RMT estimator} for confounding strength $\confRMT$ can then be defined as
\begin{align}
    \confRMT=\frac{\tauRMT\cdot\thetaRMT}{1+\tauRMT\cdot\thetaRMT}\,,\label{eq:def_rmt_estimator}
\end{align}
where $\tauRMT= (1 - \gamma) \tauPlugin$ and $\thetaRMT$ is a root of $h_{\text{RMT}}(\theta)$ if it exists and $0$ otherwise.
\end{definition}

\begin{restatable}[RMT estimator is consistent]{theorem}{consistencyRMT}\label{thm:consistency_RMT}
Let $\thetaRMT_d$ be defined as a root of $h_{\text{RMT}}(\theta)$ in some $[0, C]$ for some $C < \infty$ if it exists or $0$ otherwise. Additionally, assume that $\esdCovLim$ is not degenerate. Then, under Assumption~\ref{assumptions} with $\thetaTrue > 0$, the sequence $\{\thetaRMT_d\}$ converges a.s to $\thetaTrue.$

\end{restatable}

\section{Discussion}\label{sec:discussion}
We analyze the asymptotic behavior of the confounding strength estimator by \citet{Jan:2018} in the high-dimensional proportional regime. While the approach is consistent under population quantities, the corresponding plug-in estimator is generally biased. We correct for this bias and present a consistent estimator using tools from random matrix theory. More generally, high dimensions can help to identify the causal model, but they also warrant adapted estimators if the number of samples does not grow even faster than the dimensions.

In this work, we focus on obtaining estimators that consistently estimate the true confounding strength in the proportional asymptotic regime. An important direction for future work is to obtain non-asymptotic guarantees of convergence of the RMT estimator $\confRMT$. Obtaining convergence rates would further enhance the applicability of the RMT estimator. We leave this for future work.

Faithful estimation of confounding strength can indeed facilitate causal learning from observational data, for instance, via regularization. This has been empirically demonstrated in \citet{Jan:2019} and under the same model setting as ours, precisely characterized in \citet{Van:2022}. However, it is important to practice caution in applying such techniques more generally since causal learning or even estimation of confounding strength is a very hard problem and does require strong assumptions.



\printbibliography

\newpage
\appendix

\section{Proof of Theorem~\ref{thm:consistency_pop}}\label{app:population}
This section gives the full proof of Theorem~\ref{thm:consistency_pop} for the asymptotic behavior of the population estimator for confounding strength. We state the theorem here again for reference.
\consistencyPop*

\begin{proof}
We first show Eq.~\eqref{eq:lim_derivative_pop}. According to Eq.~\eqref{eq:logprob_pop_derivative}, the function is given by $\partial_\theta\logprobPop(\theta)=\stieltjesCov(-\theta) - \frac{1}{d}\betaStat^T\Sigma(\Sigma+\theta)^{-2}\betaStat/\frac{1}{d}\betaStat^T\Sigma(\Sigma+\theta)^{-1}\betaStat$. The first term $\stieltjesCov(-\theta)$ converges to $\stieltjesCovLim(-\theta)$ by assumption. The two quadratic forms are handled by Lemma~\ref{lem:QCT} after rewriting $\betaStat=\beta+M^{+T}\alpha=\pvector{\sigma_\alpha M^{+T} & \sigma_\beta I_d}u$ for some $u\sim\mathcal{N}(0, I_{l+d})$, which is possible because by assumption $\alpha\sim\mathcal{N}(0,\sigalpha I_l)$ and $\beta\sim\mathcal{N}(0,\sigbeta I_d)$ are independent.

\begin{align*}
    \frac{1}{d}\betaStat^T\Sigma(\Sigma+\theta)^{-1}\betaStat
    &\,=\frac{1}{d}u^T\pvector{\sigma_\alpha M^{+} \\ \sigma_\beta I_d}\Sigma(\Sigma+\theta)^{-1}\pvector{\sigma_\alpha M^{+T} & \sigma_\beta I_d}u\\
    &\asapprox\frac{1}{d}\Tr\left[\pvector{\sigma_\alpha M^{+} \\ \sigma_\beta I_d}\Sigma(\Sigma+\theta)^{-1}\pvector{\sigma_\alpha M^{+T} & \sigma_\beta I_d}\right] \tag{Lemma~\ref{lem:QCT}}\\
    &\,=\frac{1}{d}\Tr\left[\Sigma(\Sigma+\theta)^{-1}\pvector{\sigma_\alpha M^{+T} & \sigma_\beta I_d}\pvector{\sigma_\alpha M^{+} \\ \sigma_\beta I_d}\right] \tag{Trace cyclic}\\
    &\,=\frac{1}{d}\Tr\left[\Sigma(\Sigma+\theta)^{-1}(\sigalpha \Sigma^{-1}+\sigbeta I_d)\right]\tag{$\Sigma=MM^T$}\\
    &\,=\frac{\sigbeta}{d}\Tr\left[(\Sigma+\theta)^{-1}(\Sigma+\thetaTrue)\right]\tag{$\thetaTrue=\sigalpha/\sigbeta$}\\
    &\xrightarrow[d\to\infty]{a.s.} \sigbeta\E_{\lambda\sim\esdCovLim}\left[\frac{\lambda+\thetaTrue}{\lambda+\theta}\right]\,. \tag{$\esdCov\to\esdCovLim$}
\end{align*}
Similarly, we get $\frac{1}{d}\betaStat^T\Sigma(\Sigma+\theta)^{-2}\betaStat\xrightarrow[d\to\infty]{a.s.} \sigbeta\E_{\lambda\sim\esdCovLim}\left[\frac{\lambda+\thetaTrue}{(\lambda+\theta)^2}\right]$.
Plugging everything together yields
\begin{align*}
    \partial_\theta\logprobPop(\theta)
    &\xrightarrow[d\to\infty]{a.s.} \stieltjesCovLim(-\theta)-\frac{\E_{\lambda\sim\esdCovLim}\left[\frac{\lambda+\thetaTrue}{(\lambda+\theta)^2}\right]}{\E_{\lambda\sim\esdCovLim}\left[\frac{\lambda+\thetaTrue}{\lambda+\theta}\right]}\\
    &\,=\left(\stieltjesCovLim(-\theta)\cdot\E_{\lambda\sim\esdCovLim}\left[\frac{\lambda+\thetaTrue}{\lambda+\theta}\right]-\E_{\lambda\sim\esdCovLim}\left[\frac{\lambda+\thetaTrue}{(\lambda+\theta)^2}\right]\right)\E_{\lambda\sim\esdCovLim}\left[\frac{\lambda+\thetaTrue}{\lambda+\theta}\right]^{-1}\,.
\end{align*}
Using $\stieltjesCovLim(-\theta)=\E_{\lambda\sim\esdCovLim}\left[\frac{1}{\lambda+\theta}\right]$ and the identity $\frac{\lambda +\thetaTrue}{\lambda+\theta}=1-(\theta-\thetaTrue)\frac{1}{1+\lambda}$, we can simplify the first factor
\begin{align*}
    &\stieltjesCovLim(-\theta)\cdot\E_{\lambda\sim\esdCovLim}\left[\frac{\lambda+\thetaTrue}{\lambda+\theta}\right]-\E_{\lambda\sim\esdCovLim}\left[\frac{\lambda+\thetaTrue}{(\lambda+\theta)^2}\right]\\
    =&\E_{\lambda\sim\esdCovLim}\left[\frac{1}{\lambda+\theta}\right]\left(1-(\theta-\thetaTrue)\E_{\lambda\sim\esdCovLim}\left[\frac{1}{\lambda+\theta}\right]\right) - \E_{\lambda\sim\esdCovLim}\left[\frac{1}{\lambda+\theta}\right] + (\theta-\thetaTrue)\E_{\lambda\sim\esdCovLim}\left[\frac{1}{(\lambda+\theta)^2}\right]\\
    =&(\theta-\thetaTrue)\left(\E_{\lambda\sim\esdCovLim}\left[\frac{1}{(\lambda+\theta)^2}\right]-\E_{\lambda\sim\esdCovLim}\left[\frac{1}{\lambda+\theta}\right]^2\right)\\
    =&(\theta-\thetaTrue)\Var_{\lambda\sim\esdCovLim}\left[\frac{1}{\lambda+\theta}\right]\,.
\end{align*}
This concludes the first part of the proof.\\

For the second statement, first observe that the almost sure convergence in Eq.~\eqref{eq:lim_derivative_pop} for each $\theta\geq 0$ implies that this convergence also holds almost surely on a countable set such as $[0,C]\cap\Q$. Since each function $\partial_\theta\logprobPop_d$ is analytic and bounded on $[0,C]$, we can further upgrade Eq.~\eqref{eq:lim_derivative_pop} to almost surely uniform convergence on $[0,C]$ by Vitali's convergence theorem.
Now let $(\thetaPop_d)_{d\in\N}$ be a sequence of roots as described in the theorem and let $F^{pop}(\theta)$ denote the function on the right hand side of Eq.~\eqref{eq:lim_derivative_pop}. First note that the functions $\partial_\theta\logprobPop_d$ eventually have a root $\thetaPop_d$ in $[0,C]$ with probability 1: since $\thetaTrue<C$, there exist $\theta_-,\theta_+$ with $0<\theta_-<\thetaTrue<\theta_+<C$ with $F^{pop}(\theta_-)<0$ and $F^{pop}(\theta_+)>0$. The convergence of the functions $\partial_\theta\logprobPop_d$ then implies that $\partial_\theta\logprobPop_d(\theta_-)<0$ and $\partial_\theta\logprobPop_d(\theta_+)>0$ eventually. Since $\partial_\theta\logprobPop_d$ is continuous, the intermediate value theorem then implies the existence of a root in $(\theta_-,\theta_+)\subset[0,C]$. The proof is concluded with Hurwitz's theorem, which states that the sequence of roots $(\thetaPop_d)_{d\in\N}$ of analytic functions converges to the unique root $\thetaTrue$ of the limiting function.
%

\end{proof}

\section{Proof of Theorem~\ref{thm:inconsistency_plugin}}\label{app:plugin}
For the proof of Theorem~\ref{thm:inconsistency_plugin} about the asymptotic behavior of the plug-in estimator, we require additional technical statements. The first characterizes the asymptotic behavior of the statistical noise for our causal model.
\begin{lemma}[Asymptotics of the statistical noise]\label{lem:statistical_noise_limit}
Under Assumption~\ref{assumptions}, the statistical noise $\sigmastat$ concentrates as
\begin{align*}
    \frac{\sigmastat}{d}- (\tilde{\gamma}-1)\sigalpha \xrightarrow[d\to\infty]{a.s.} 0\,.
\end{align*}
\end{lemma}
\begin{proof}
According to Proposition~\ref{lem:observational_dist}, the statistical noise is given by $\sigmastat=\sigmacaus+\norm{\alpha}_{I_l-M^+M}^2$. The term $\sigmacaus$ is assumed to be constant, but the quadratic form $\norm{\alpha}_{I_l-M^+M}^2$ grows with $d$ and is controlled by Lemma~\ref{lem:QCT} as
\begin{align*}
    \frac{\sigmastat}{d}
    =\frac{\sigmacaus}{d} + \frac{1}{d}\alpha^T(I_l-M^+M)\alpha 
    \asapprox \frac{\Tr(I_l-M^+M)}{d} \sigalpha
    &= \frac{(l-\Tr(MM^+))}{d} \sigalpha\\
    &=\frac{l-d}{d}\sigalpha\\
    &=(\tilde{\gamma}-1)\sigalpha\,.
\end{align*}
\end{proof}

The second technical lemma covers the asymptotic behavior of traces that involve both the sample covariance matrix $\empcov$ and the population covariance matrix $\Sigma$:

\begin{lemma}[Asymptotics of mixed terms]\label{lem:mixed_traces}
Under Assumption~\ref{assumptions}, it holds for any $\theta\geq 0$ that
\begin{align*}
    \frac{1}{d}\Tr\left[\left(\empcov+\theta\right)^{-1}\empcov\Sigma^+\right] &\xrightarrow[d\to\infty]{a.s.} \gamma\theta m(-\theta)^2 + (1-\gamma)m(-\theta)
    \intertext{and}
    \frac{1}{d}\Tr\left[\left(\empcov+\theta\right)^{-2}\empcov\Sigma^+\right] &\xrightarrow[d\to\infty]{a.s.}-\gamma m(-\theta)^2 +2\gamma\theta m(-\theta)M(-\theta) + (1-\gamma)M(-\theta)\,,
\end{align*}
where $m(-\theta)=\E_{\lambda\sim\mu}\left[\frac{1}{\lambda+\theta}\right]$ and $M(-\theta)=\E_{\lambda\sim\mu}\left[\frac{1}{(\lambda+\theta)^2}\right]$.
\end{lemma}
\begin{proof}
The asymptotic behavior of these quadratic forms is not covered by Theorem~\ref{thm:sample_covariance_matrix}, because the dependencies between $\empcov$ and $\Sigma$ create complications. To treat these we require an additional result by \citet{Led:2011} combined with Vitali's convergence theorem which, in our notation, states that
\begin{align*}
    \frac{1}{d}\Tr\left((\empcov - z)^{-1}g(\Sigma)\right)\xrightarrow[d\to\infty]{a.s.}-\frac{1}{z}\E_{\lambda\sim\nu}\left[\frac{g(\lambda)}{\stieltjesEmpkerLim(z)\lambda + 1}\right]\,.
\end{align*}
We first use this result to obtain the limit for $\frac{1}{d}\Tr\left((\empcov-z)^{-1}\Sigma^+\right)$ by considering $g(\lambda)=1/\lambda$ and the identity 
\begin{align*}
    -\frac{1}{z\lambda}\frac{1}{\stieltjesEmpkerLim(z)\lambda+1}=\frac{1}{z}\left(\frac{1}{\lambda-\left(-\frac{1}{\stieltjesEmpkerLim(z)}\right)}-\frac{1}{\lambda}\right)\,,
\end{align*}
which yields
\begin{align*}
    \frac{1}{d}\Tr\left((\empcov-z)^{-1}\Sigma^+\right)\xrightarrow[d\to\infty]{a.s.}
    \E_{\lambda\sim\nu}\left[-\frac{1}{z\lambda}\frac{1}{\stieltjesEmpkerLim(z)\lambda + 1}\right]
    = \frac{1}{z}\stieltjesCovLim\left(-\frac{1}{\stieltjesEmpkerLim(z)}\right) - \frac{1}{z}\stieltjesCovLim(0)\,,
\end{align*}
where we recall that $\stieltjesCovLim(z)=\E_{\lambda\sim\nu}[\frac{1}{\lambda-z}]$. To relate the population Stieltjes transform $\stieltjesCovLim$ back to the sample Stieltjes transforms $\stieltjesEmpcovLim$ and $\stieltjesEmpkerLim$, we can use the identities from Theorem~\ref{thm:sample_covariance_matrix} to obtain
\begin{align*}
    \frac{1}{d}\Tr\left((\empcov-z)^{-1}\Sigma^+\right)\xrightarrow[d\to\infty]{a.s.}&
    -\gamma \stieltjesEmpcovLim(z)\stieltjesEmpkerLim(z) - \frac{1}{z}\stieltjesCovLim(0) \tag{Eq.~\eqref{eq:limiting_stieltjes}}\\
    =&-\gamma \stieltjesEmpcovLim(z)^2 + \frac{1-\gamma}{z}\stieltjesEmpcovLim(z) - \frac{1}{z}\stieltjesCovLim(0) \tag{Eq.~\eqref{eq:stieltjes_cov_ker}}\,.
\end{align*}
Evaluating the above expression at $z=-\theta$ then yields
\begin{align*}
    \frac{1}{d}\Tr\left((\empcov+\theta)^{-1}\Sigma^+\right) \xrightarrow[d\to\infty]{a.s.}
    -\gamma \stieltjesEmpcovLim(-\theta)^2 - \frac{1-\gamma}{\theta}\stieltjesEmpcovLim(-\theta) + \frac{1}{\theta}\stieltjesCovLim(0)\,.
\end{align*}
All that remains is to relate $(\empcov+\theta)^{-1}\Sigma^+$ to the terms we are interested in. Using the identity $(\empcov+\theta)^{-1}\empcov=I-\theta(\empcov+\theta)^{-1}$, we get the first statement of this lemma
\begin{align*}
    \frac{1}{d}\Tr\left[\left(\empcov+\theta\right)^{-1}\empcov\Sigma^+\right]
    =&\frac{1}{d}\Tr\left[\Sigma^+\right] - \theta \frac{1}{d}\Tr\left[\left(\empcov+\theta\right)^{-1}\Sigma^+\right]\\
    \xrightarrow[d\to\infty]{a.s.}& \stieltjesCovLim(0) -\theta\left(-\gamma \stieltjesEmpcovLim(-\theta)^2 - \frac{1-\gamma}{\theta}\stieltjesEmpcovLim(-\theta) + \frac{1}{\theta}\stieltjesCovLim(0)\right)\\
    =& \gamma\theta m(-\theta)^2 + (1-\gamma)m(-\theta)\,.
\end{align*}
The second statement of this lemma also follows directly by taking the derivative, which can be exchanged with the limit $d\to\infty$ using similar arguments as in the main paper after Proposition~\ref{thm:ConsistentStieltjes}, to obtain
\begin{align*}
    \frac{1}{d}\Tr\left[\left(\empcov+\theta\right)^{-2}\empcov\Sigma^+\right]
    =&-\partial_\theta \frac{1}{d}\Tr\left[\left(\empcov+\theta\right)^{-1}\empcov\Sigma^+\right]\\
    \xrightarrow[d\to\infty]{a.s.}~&-\partial_\theta \left(\gamma\theta m(-\theta)^2 + (1-\gamma)m(-\theta)\right)\\
    =&-\gamma m(-\theta)^2 +2\gamma\theta m(-\theta)M(-\theta) + (1-\gamma)M(-\theta)\,,
\end{align*}
where the last step used $\partial_\theta m(-\theta)=M(-\theta)$.
\end{proof}

We are now ready to give the full proof of Theorem~\ref{thm:inconsistency_plugin}.

\inconsistencyPlugin*
\begin{proof}
We first show Eq.~\eqref{eq:lim_derivative_plugin}.
This proof for the plug-in quantities $\empcov, \betaMinNorm$ follows the same strategy as the proof of Theorem~\ref{thm:consistency_pop} for $\Sigma, \betaStat$, but additional complications arise because $\betaMinNorm$ asymptotically depends on both the population term $M$ and the empirical quantities.
Similarly as for $\betaStat$, we treat $\betaMinNorm$ by combining the equations $\betaMinNorm=(XX^T)^+XY$, $Y=X^T\betaStat+E$ for $E\sim\mathcal{N}(0,\sigmastat I_n)$, and $\betaStat=\beta+M^{+T}\alpha$ to obtain
\begin{align*}
    \betaMinNorm = \pvector{\sigma_\alpha M^{+T} & \sigma_\beta I_d & \tilde{\sigma}(XX^T)^+X}v\quad\text{for some }v\sim\mathcal{N}(0, I_{l+d+n})\,.
\end{align*} 

As before, we get for $k\in\{1,2\}$ that

\begin{align*}
    &\hspace{15pt}\frac{1}{d}\betaMinNorm^T\empcov(\empcov+\theta)^{-k}\betaMinNorm\\
    &\,=\frac{1}{d}v^T\pvector{\sigma_\alpha M^{+} \\ \sigma_\beta I_d \\ \tilde{\sigma}X^T(XX^T)^+}\empcov(\empcov+\theta)^{-k}\pvector{\sigma_\alpha M^{+T} & \sigma_\beta I_d & \tilde{\sigma}(XX^T)^+X}v\\
    &\asapprox\frac{1}{d}\Tr\left[\pvector{\sigma_\alpha M^{+} \\ \sigma_\beta I_d \\ \tilde{\sigma}X^T(XX^T)^+}\empcov(\empcov+\theta)^{-k}\pvector{\sigma_\alpha M^{+T} & \sigma_\beta I_d & \tilde{\sigma}(XX^T)^+X}\right] \tag{Lemma~\ref{lem:QCT}}\\
    &\,=\frac{1}{d}\Tr\left[\empcov(\empcov+\theta)^{-k}\pvector{\sigma_\alpha M^{+T} & \sigma_\beta I_d & \tilde{\sigma}(XX^T)^+X}\pvector{\sigma_\alpha M^{+} \\ \sigma_\beta I_d \\ \tilde{\sigma}X^T(XX^T)^+}\right] \tag{Trace cyclic}\\
    &\,=\frac{1}{d}\Tr\left[\empcov(\empcov+\theta)^{-k}\left(\sigalpha\Sigma^++\sigbeta I_d+\frac{\sigmastat}{n} \empcov^{-1}\right)\right]\\
    &\,=\frac{1}{d}\Tr\left[\empcov(\empcov+\theta)^{-k}\left(\sigalpha\Sigma^++\sigbeta I_d+\gamma(\tilde{\gamma}-1)\sigalpha \empcov^{-1}\right)\right]\tag{Lemma~\ref{lem:statistical_noise_limit}}\\
    &\,=\frac{\sigbeta}{d}\Tr\left[(\empcov+\theta)^{-k}(\empcov+\gamma(\tilde{\gamma}-1)\thetaTrue)\right] + \thetaTrue \frac{\sigbeta}{d}\Tr\left[(\empcov+\theta)^{-k}\empcov\Sigma^+\right]\,.
\end{align*}
The second term contains both the population term $\Sigma$ and the sample term $\empcov$, which is treated separately in Lemma~\ref{lem:mixed_traces}. For readability, we use the shorthand notation $m=\E_{\lambda\sim\mu}\left[\frac{1}{\lambda+\theta}\right]$ and $M=\E_{\lambda\sim\mu}\left[\frac{1}{(\lambda+\theta)^2}\right]$, under which the limit for the first term is given by
\begin{align*}
    \frac{1}{d}\Tr\left[(\empcov+\theta)^{-k}(\empcov+\gamma(\tilde{\gamma}-1)\thetaTrue)\right]\xrightarrow[d\to\infty]{a.s.}
    \begin{cases}
        1-\theta m + \gamma(\tilde{\gamma}-1)\thetaTrue m, & \text{for }k=1 \\
        m-\theta M + \gamma(\tilde{\gamma}-1)\thetaTrue M, & \text{for }k=2
    \end{cases}\,.
\end{align*}
Combined with Lemma~\ref{lem:mixed_traces}, this yields
\begin{align*}
    \frac{1}{d}\betaMinNorm^T\empcov(\empcov+\theta)^{-k}\betaMinNorm\xrightarrow[d\to\infty]{a.s.}
    \begin{cases}
        1-\theta m + \thetaTrue(\gamma\theta m^2 + (1-2\gamma+\gamma\tilde{\gamma})m), & \text{for }k=1 \\
        m-\theta M + \thetaTrue (-\gamma m^2 +2\gamma\theta mM + (1-2\gamma+\gamma\tilde{\gamma})M), & \text{for }k=2
    \end{cases}\,.    
\end{align*}
Together with $m_{\empcov}(-\theta)\xrightarrow[d\to\infty]{a.s.}m$, this covers the individual components of $\partial_\theta\logprobPlugin(\theta)=\stieltjesEmpcov(-\theta) - \frac{1}{d}\betaMinNorm^T\empcov(\empcov+\theta)^{-2}\betaMinNorm / \frac{1}{d}\betaMinNorm^T\empcov(\empcov+\theta)^{-1}\betaMinNorm$. It remains to plug everything in, which we do after factoring out the denominator $\frac{1}{d}\betaMinNorm^T\empcov(\empcov+\theta)^{-1}\betaMinNorm$ to obtain
\begin{align*}
    &\stieltjesEmpcov(-\theta)\cdot \frac{1}{d}\betaMinNorm^T\empcov(\empcov+\theta)^{-1}\betaMinNorm - \frac{1}{d}\betaMinNorm^T\empcov(\empcov+\theta)^{-2}\betaMinNorm\\
    \xrightarrow[d\to\infty]{a.s.}& m \cdot \left[1-\theta m + \thetaTrue(\gamma\theta m^2 + (1-2\gamma+\gamma\tilde{\gamma})m)\right] - \left[m-\theta M + \thetaTrue (-\gamma m^2 +2\gamma\theta mM + (1-2\gamma+\gamma\tilde{\gamma})M)\right]\\
    =&(\theta - (1-2\gamma+\gamma\tilde{\gamma})\thetaTrue)\cdot(M-m^2) + \gamma\thetaTrue \left(\theta m^3 + m^2 - 2\theta m M\right)\\
    =&(\theta - (1-2\gamma+\gamma\tilde{\gamma})\thetaTrue)\cdot(M-m^2) +\gamma\thetaTrue (2m^2-2M - (1-\theta m)m^2 + 2(1-\theta m )M)\\
    =&(\theta - (1+\gamma\tilde{\gamma})\thetaTrue)\cdot(M-m^2) + \gamma\thetaTrue (1-\theta m)(2M-m^2)\\
    =&\left[\theta - (1+\gamma\tilde{\gamma})\thetaTrue + \gamma\thetaTrue(1-\theta m)(1+\frac{M}{M-m^2})\right]\cdot (M-m^2)\,,
\end{align*}
which concludes the first part of the proof.\\

For the second statement, observe that Eq.~\eqref{eq:plugin_consistency_cond} is equivalent to $F^{plg}(\thetaTrue)=0$, where $F^{plg}$ is the function on the right hand side of Eq.~\eqref{eq:lim_derivative_plugin}. The assumption in this theorem therefore states that $F^{plg}(\thetaTrue)\neq 0$. 
Let $(\thetaPlugin_d)_{d\in\N}$ be the sequence described in the theorem. In the case where $\partial_\theta\logprobPlugin_d$ does not have a root infinitely often, we have $\thetaPlugin_d=0$ infinitely often and therefore $\thetaPlugin_d\not\to \thetaTrue$ as $d\to\infty$ since $\thetaTrue\neq 0$. Therefore, now assume that $\thetaPlugin_d$ is a root of $\partial_\theta\logprobPlugin_d$ eventually.
Assume that the claim is false, that is, $\thetaPlugin_d\xrightarrow[d\to\infty]{}\thetaTrue$ with positive probability. 
Similarly to the proof of Theorem~\ref{thm:consistency_pop}, we get that the convergence in Eq.~\eqref{eq:lim_derivative_plugin} holds almost surely uniformly on $[0,C]$ for some $C>\thetaTrue$. The convergence $\thetaPlugin_d\to\thetaTrue$ also implies that $\thetaPlugin_d\in[0,C]$ eventually. Putting everything together, we get for sufficiently large $d$ that 
\begin{align*}
    \abs{F^{plg}(\thetaTrue)}
    &=\abs{F^{plg}(\thetaTrue) - \partial_\theta\logprobPlugin_d(\thetaPlugin_d)} \tag{$\partial_\theta\logprobPlugin_d(\thetaPlugin_d)=0$}\\
    &\leq \abs{\partial_\theta\logprobPlugin_d(\thetaPlugin_d) - F^{plg}(\thetaPlugin_d)} + \abs{F^{plg}(\thetaPlugin_d) - F^{plg}(\thetaTrue)}\\
    &\leq \sup_{\theta\in[0,C]}\abs{\partial_\theta\logprobPlugin_d(\theta) - F^{plg}(\theta)} + \abs{F^{plg}(\thetaPlugin_d) - F^{plg}(\thetaTrue)}\\
    &\xrightarrow[d\to\infty]{}0\,,
\end{align*}
where the first summand goes to 0 by uniform convergence and the second summand goes to 0 by continuity of $F^{plg}$ and $\thetaPlugin_d\to\thetaTrue$. This implies $F^{plg}(\thetaTrue)=0$, which is a contradiction.
\end{proof}

\section{RMT consistent estimators for quantitites of interest}
\label{sec:rmt_proofs}





\begin{theorem}[\textbf{Consistent estimation of statistical noise}]
\label{thm:statisticalNoise}
Under the model in Eq.~\eqref{eq:causal_model}, 
\begin{equation*}
    \frac{1}{1-\gamma}\frac{\norm{ Y}^2_{I-X^+X}}{nd} - \frac{\sigmastat}{d} \stackrel{a.s}{\longrightarrow} 0\,.
\end{equation*}

\end{theorem}

\begin{proof}

\begin{align*}
    \frac{1}{n d}\norm{Y}^2 = \frac{1}{n d}\norm{X \betaStat + E}^2 = \frac{1}{n d}\betaStat^T X^T X \betaStat + \frac{1}{n d} E^T E + \frac{2}{n d} \betaStat^T X^T E.
\end{align*}

We know that the minimum $l_2$ norm estimator admits a following closed form solution given by $\betaMinNorm = (X^TX)^+X^TY = (X^TX)^+X^T(X\betaStat + E) \stackrel{w.h.p}{=} \betaStat + (X^TX)^+X^TE$, where we used the fact that $\textrm{rank}(X^TX) = d \; \; \textrm{ w.h.p}$ to arrive at the last equality. Letting $\kappa = (X^TX)^+X^TE$, we have

\begin{align*}
    \frac{1}{n d}\betaMinNorm^T X^T X \betaMinNorm &= \frac{1}{n d}(\betaStat + \kappa)^T X^TX (\betaStat + \kappa), \\
    &= \frac{1}{n d} \betaStat^T X^T X \betaStat  + \frac{1}{n d} \kappa^T X^T X \kappa + \frac{2}{n d} \betaStat^T X^T X \kappa.
\end{align*}
From the closed form expression for $\betaMinNorm$, 

\begin{align*}
     \frac{1}{n d}\betaMinNorm^T X^T X \betaMinNorm &= \frac{1}{n d} Y^T X (X^T X)^+ X^T X (X^T X)^+ X^T Y, \\
     &= \frac{1}{n d} Y^T X (X^T X)^+ X^T Y, \\
     &= \frac{1}{n d} Y^T X X^+ Y.
\end{align*}

Similarly substituting $\kappa = (X^TX)^+X^TE$, we have
\begin{align*}
     \frac{1}{n d}\kappa^T X^T X \kappa &= \frac{1}{n d} E^T X (X^T X)^+ X^T X (X^T X)^+ X^T E, \\
     &= \frac{1}{n d} E^T X (X^T X)^+ X^T E, \\
     &= \frac{1}{n d} E^T X X^+ E, \\
     &= \frac{\gamma \sigmastat }{d } + \mathcal{O}(1/\sqrt{d}).
\end{align*}
To derive the last equality, we first apply Lemma \ref{lem:QCT} to show that $\frac{1}{n d} E^T X X^+ E = \frac{\sigmastat }{nd } \Tr[X X^+] + \mathcal{O}(1 / \sqrt{p}).$ The equality follows using $\Tr[A A^+] = rank(A)$ for any $A \in \R^{n \times d}$ and $$\frac{1}{n d} E^T X X^+ E = \frac{\gamma \sigmastat }{d } + \mathcal{O}(1/\sqrt{d}).$$

Now let us consider the term $\frac{2}{n d} \betaStat^T X^T X \kappa$.

\begin{align*}
    \frac{2}{n d} \betaStat^T X^T X \kappa &= \frac{2}{n d} \betaStat^T X^T X (X^TX)^+ X^T E, \\
    &= \frac{2}{n d} \betaStat^T X^T E \rightarrow 0  \; \textrm{as } d \rightarrow \infty \quad \quad \textrm{(Hoeffding's inequality)}
\end{align*}

Following similar arguments, we have
\begin{align*}
    \frac{1}{n d} E^T E &= \frac{\sigmastat}{d} + \mathcal{O}(\frac{1}{d \sqrt{n}})
\end{align*}

Putting everything together, we have 
\begin{align*}
    \frac{1}{n d}\norm{Y}^2 &= \frac{1}{n d} Y^T X X^+ Y -  \frac{\gamma \sigmastat }{d } + \frac{\sigmastat}{d} + \mathcal{O}(1 / \sqrt{d}) \\
    \frac{\sigmastat}{d} &= \frac{1}{(1 - \gamma) n d}\norm{Y}^2_{I - X X^+} + \mathcal{O}(1 / \sqrt{d}).
\end{align*}

\end{proof}

\begin{lemma}[\textbf{Asymptotics of quadratic form with a deterministic sequence}]
\label{lemma:QuadraticFormDeterministic}
For any $\theta \in \R^+$, let $\eta$ be the unique solution in $\mathbb{R}^{-}$ satisfying $\tilde{m}(\eta) = 1/\theta$. Then, for any deterministic sequence of vectors $\mycurls{v_d}$ with uniformly bounded (Euclidean) norm, as $d, n \rightarrow \infty$ such that $d/n \rightarrow \gamma \in (0, 1)$ , 

\begin{equation*}
    \scalprod{v_d}{\empcov (\empcov - \eta)^{-1} v_d} - \scalprod{v_d}{\Sigma (\Sigma + \theta)^{-1} v_d} \longrightarrow 0.
\end{equation*}

\end{lemma}

\begin{proof}
Observe that for any $\eta < 0$, 
\begin{align*}
    \scalprod{v_d}{\empcov (\empcov - \eta)^{-1} v_d} = \norm{v_d}^2 - \scalprod{v_d}{ (\empcov - \eta)^{-1} v_d}.
\end{align*}
The result follows from the Generalized Marchenko Pastur Theorem \citep{Sil:1995}, which states that for any $\theta \in \R^+$,
\begin{equation*}
    \scalprod{v_d}{(\empcov - \eta)^{-1} v_d} - \scalprod{v_d}{(\Sigma + \theta)^{-1} v_d} \longrightarrow 0.
\end{equation*}
\end{proof}

\consistentQuadForm*

\begin{proof}

Let $\eta$ be the unique solution in $\R^-$ satisfying $\tilde{m}(\eta) = 1/\theta$. From Lemma \ref{lemma:QuadraticFormDeterministic}, we have for any $\theta \in \R^+$, as $n, d \rightarrow \infty$ such that $d / n \rightarrow \gamma \in (0, 1)$, 
\begin{equation}
    \scalprod{\frac{\betaStat}{\norm{\betaStat}}}{\empcov (\empcov - \eta)^{-1} \frac{\betaStat}{\norm{\betaStat}}} - \scalprod{\frac{\betaStat}{\norm{\betaStat}}}{\Sigma (\Sigma + \theta)^{-1} \frac{\betaStat}{\norm{\betaStat}}} \stackrel{a.s}{\longrightarrow} 0
\end{equation}

Therefore, it suffices to consistently estimate $\scalprod{\frac{\betaStat}{\norm{\betaStat}}}{\empcov (\empcov - \eta)^{-1} \frac{\betaStat}{\norm{\betaStat}}}$. First, we characterize the asymptotic behavior of $\frac{1}{d} \scalprod{\betaMinNorm}{\empcov (\empcov - \eta)^{-1}\betaMinNorm}$, where $\betaMinNorm = \betaStat + \sigmastat (X X^T)^+ XE$, where $E \sim \Gauss{0}{I_n}$. 
\begin{multline*}
   \frac{1}{d} \scalprod{\betaMinNorm}{\empcov (\empcov - \eta)^{-1}\betaMinNorm} =
    \frac{1}{d}\scalprod{\betaStat}{\empcov (\empcov - \eta)^{-1}\betaStat} + \frac{2\sigmastat}{d}\betaStat^T\empcov (\empcov - \eta)^{-1}(XX^T)^+XE + \\ \frac{\sigmastat}{d}E^TX^T(XX^T)^+\empcov (\empcov - \eta)^{-1}(XX^T)^+XE\,.
\end{multline*}

The first term in the expansion resembles the quantity of interest. 

For the second term, notice that, since $E \sim \Gauss{0}{I_n}$, 
\begin{align*}
    \frac{2\sigmastat}{d}\betaStat^T \empcov(\empcov - \eta)^{-1}(XX^T)^+XE\sim\mathcal{N}(0, \lVert \frac{2\sigmastat}{d}X^T(XX^T)^+ \empcov(\empcov - \eta)^{-1}\betaStat\rVert^2),
\end{align*}

where 
\begin{align*}
    \lnorm{\frac{2\sigmastat}{d}X^T(XX^T)^+ \empcov(\empcov - \eta)^{-1}\betaStat}^2
&=\frac{4\sigmastat}{d^2}\betaStat^T\empcov(\empcov - \eta)^{-1}(XX^T)^+XX^T(XX^T)^+\empcov(\empcov - \eta)^{-1}\betaStat \\
&=\frac{4\sigmastat}{d^2n}\betaStat^T \empcov(\empcov - \eta)^{-1}\empcov^+\empcov(\empcov - \eta)^{-1}\betaStat \\
& \stackrel{a.s}{\longrightarrow} 0
\,.
\end{align*}
Therefore, the second term vanishes. For the last expression,
\begin{align*}
    \frac{\sigmastat}{d}E^TX^T(XX^T)^+\empcov(\empcov - \eta)^{-1}(XX^T)^+XE
&= \frac{\sigmastat}{d}\frac{1}{n^2}E^TX^T\hat{\Sigma}^+\empcov(\empcov - \eta)^{-1}\hat{\Sigma}^+XE \\
& \stackrel{a.s}{\longrightarrow} \frac{\sigmastat}{d}\frac{1}{n}\text{tr}\left(\hat{\Sigma}^+\empcov(\empcov - \eta)^{-1}\right) \quad \textrm{(Trace Lemma, conditioned on $X$)} \\
& \stackrel{a.s}{\longrightarrow} \gamma\frac{\sigmastat}{d}m(\eta).
\end{align*}
From Theorem \ref{thm:sample_covariance_matrix}, we know that 
\begin{equation*}
    m(\eta)=\frac{1}{\gamma}\left(\tilde{m}(\eta)+\frac{1-\gamma}{\eta}\right)=\frac{1}{\gamma}\left(\frac{1}{\theta}+\frac{1-\gamma}{\eta}\right).
\end{equation*}
Therefore, 
\begin{equation*}
     \frac{\sigmastat}{d}E^TX^T(XX^T)^+\empcov(\empcov - \eta)^{-1}(XX^T)^+XE - \frac{\sigmastat}{d}\left(\frac{1}{\theta}+\frac{1-\gamma}{\eta}\right) \stackrel{a.s}{\longrightarrow} 0.
\end{equation*}
Following the same arguments, it is easy to verify that
\begin{equation*}
    \frac{1}{d}\norm{\betaMinNorm}^2-\frac{\sigmastat}{d}\gamma m(0) - \frac{1}{d}\norm{\betaStat}^2 \stackrel{a.s}{\longrightarrow} 0\,.
\end{equation*}
Combining the estimators with the result from Theorem \ref{thm:statisticalNoise}, we have the desired result. 
\end{proof}

\consistencyRMT*
\begin{proof}
The proof follows following the same arguments as in the proof of \ref{thm:consistency_pop}.
\end{proof}

\end{document}